\documentclass[preprint]{elsarticle}
\usepackage{ifpdf}
\usepackage{graphicx,natbib,amssymb,lineno,amsmath,algorithm,algorithmic,geometry}
\usepackage{overpic}
\usepackage{multirow}
\usepackage{subfigure}
\usepackage{times}
\usepackage{xcolor}
\usepackage{soul}
\usepackage[utf8]{inputenc}
\usepackage[small]{caption}
\usepackage{amssymb}
\usepackage{amsmath}
\usepackage{times}
\usepackage{graphicx}
\usepackage{amssymb}
\usepackage{amsthm}

\usepackage{algorithm}
\usepackage{algorithmic}
\usepackage{graphicx}
\usepackage{cases}
\usepackage{multirow}
\usepackage{hyperref}

\usepackage{amsmath}
\usepackage{booktabs}
\usepackage{subfigure}
\usepackage{cases}
\usepackage{enumerate}

\usepackage{picinpar}

\usepackage{graphicx}
\usepackage{amsfonts}
\usepackage{amsthm}
\usepackage{amsmath}
\usepackage{amssymb}
\usepackage{multirow}
\usepackage{algorithmic}
\usepackage{algorithm}
\usepackage{float}
\usepackage{makecell}
\usepackage{tabularx}
\newcolumntype{Y}{>{\centering\arraybackslash}X}

\newcommand{\comment}[1]{}

\DeclareMathOperator*{\argmin}{argmin}

\newtheorem{prop}{Proposition}

\newtheorem{thm}{Theorem}
\newtheorem{lem}{Lemma}
\newtheorem{problem}{Problem}
\def\*#1{\mathbf{#1}}
\def\^#1{\mathcal{#1}}

\journal{Neurocomputing}

\begin{document}

\begin{frontmatter}
\title{Matrix Recovery with Implicitly Low-Rank Data}

\author[a]{Xingyu Xie}
\ead{nuaaxing@gmail.com}
\author[b]{Jianlong Wu}
\ead{jlwu1992@pku.edu.cn}
\author[c]{Guangcan Liu}
\ead{gcliu@nuist.edu.cn}
\author[a]{Jun Wang\corref{cor}}
\ead{junwang@outlook.com}

\cortext[cor]{Corresponding author.}

\address[a]{College of Mechanical and Electrical Engineering, Nanjing University of Aeronautics and Astronautics, Nanjing 210016, P. R. China}
\address[b]{Key Laboratory of Machine Perception (MOE), School of EECS, Peking University, Beijing 100871, P. R. China}
\address[c]{B-DAT Laboratory, Department of Information and	Control, \\ Nanjing University of Information Science and Technology, Nanjing	210014, P. R. China}

\begin{abstract}
In this paper, we study the problem of matrix recovery, which aims to restore a target matrix of authentic samples from grossly corrupted observations. Most of the existing methods, such as the well-known Robust Principal Component Analysis (RPCA), assume that the target matrix we wish to recover is low-rank. However, the underlying data structure is often non-linear in practice, therefore the low-rankness assumption could be violated. To tackle this issue, we propose a novel method for matrix recovery in this paper, which could well handle the case where the target matrix is low-rank in an implicit feature space but high-rank or even full-rank in its original form. Namely, our method pursues the low-rank structure of the target matrix in an implicit feature space. By making use of the specifics of an accelerated proximal gradient based optimization algorithm, the proposed method could recover the target matrix with non-linear structures from its corrupted version. Comprehensive experiments on both synthetic and real datasets demonstrate the superiority of our method.
\end{abstract}


\end{frontmatter}

\section{Introduction}

Due to the unconstrained nature of today's data acquisition procedure, the observed data is often contaminated by gross errors, such as large corruptions and outliers. The gross errors, in general, could significantly reduce the representativeness of data samples and therefore seriously distort the analysis of data. Given this pressing situation, it is of considerable practical significance to study the problem of \emph{Matrix Recovery}, which aims to correct the errors possibly existing in a data matrix of observations.
\begin{problem}\label{NLRPCA}
	(Matrix Recovery). Let $\mathbf{X}\in\mathbb{R}^{d\times n}$ be an observed data matrix which could be decomposed as
	\begin{align*}
		\mathbf{X} = \mathbf{A} + \mathbf{E},
	\end{align*}
	where $\mathbf{A}$ is the target matrix of interest in which each column is a $d$-dimensional authentic sample, and $\mathbf{E}$ corresponds to the possible errors. Given $\mathbf{X}$, the goal is to recover $\mathbf{A}$.
\end{problem}
In general, the above problem is ill-posed, and thus some restrictions are necessary to be imposed for both $\mathbf{A}$ and $\mathbf{E}$.
Some methods have already been proposed to solve the above problem with proper constraints.
For example, provided that $\mathbf{A}$ is low-rank and $\mathbf{E}$ is sparse, Problem~\ref{NLRPCA} could be well solved by a convex procedure termed Principal Component Pursuit (PCP), which is also known as Robust Principal Component Anlaysis (RPCA)~\cite{candes2011robust,zhang2015exact}. Outlier Pursuit (OP)~\cite{xu2010robust} solves Problem~\ref{NLRPCA} under the conditions that $\mathbf{A}$ is low-rank and $\mathbf{E}$ is column-wisely sparse. Under similar conditions, Low-Rank Representation (LRR)~\cite{tpami_2013_lrr,liu:tpami:2016} guarantees to recover the row space of $\mathbf{A}$. In addition, LRR equipped with proper dictionaries could handle the cases where $\mathbf{A}$ is of high coherence~\cite{liu2017blessing,liu:tsp:2016}.
Even though these related approaches are very powerful, they all rely on the assumption that $\mathbf{A}$ is low-rank, which, however, could be violated in practice.
\par
To cope with the data of complex structures, it would be more suitable to consider the cases where $\mathbf{A}$ is low-rank after feature mapping: namely, $\mathbf{A}$ is implicitly low-rank in some (unknown) feature space but could be high-rank or even full-rank by itself. There are only a few investigations in this direction, such as the Kernel Principal Component Analysis (KPCA)~\cite{nguyen2009robust}. In general, KPCA could apply to the data matrix that is implicitly low-rank but originally high-rank. However, this method assumes that the data is contaminated by small Gaussian noise and is therefore brittle in the presence of gross errors. We also notice that many kernel methods have been established in community of low-rankness modeling, e.g., ~\cite{xiao2016robust,ji2017low,xie2018implicit,nguyen2015kernel}. Nevertheless, these methods are designed for a specific purpose of classification or clustering, and thus they cannot be directly applied to Problem \ref{NLRPCA} which is essentially a data recovery problem.
\par
In this work, we would like to study Problem~\ref{NLRPCA} in the context of $\mathbf{A}$ is implicitly low-rank and $\mathbf{E}$ contains gross errors. Following~\cite{xu2010robust,tpami_2013_lrr}, we will focus on the case where $\mathbf{E}$ is column-wisely sparse, i.e., the observed data matrix $\mathbf{X}$ is contaminated by outliers. The basic idea of our method to pursue the low-rank structure of $\mathbf{A}$ in an implicit feature space of higher but unknown (maybe infinite) dimension is simple and traditional. Nevertheless, it is indeed rather challenging to realize this idea:
\begin{itemize}
	\item Firstly, the rank of an unknown-dimensional matrix cannot be calculated directly. To overcome this difficulty, we show that the nuclear norm of $\mathbf{A}$ after feature mapping is actually equivalent to the nuclear norm of the square root of the Gram matrix (or kernel matrix as equal). This enables the possibility of obtaining a computable formulation for solving the low-rank constraint in an unknown-dimensional implicit space.
	\item Secondly, in the presence of outliers, it is actually inaccurate to estimate the Gram matrix based on $\mathbf{X}$. This is because the outliers could seriously reduce the quality of the estimated Gram matrix, in which the whole column and row corresponding to the outliers are corrupted. Hence, we build our algorithm upon a kernel function that only defines the inner product of the points in feature space. Since the kernel function is independent of the data $\mathbf{X}$, this strategy is conducive to reduce the influence of outliers and preserve the geometry structure of the clean data.
	\item Finally, the combination of implicit feature mapping with kernel low-rankness pursuit generally leads to a challenging optimization problem, which is nonconvex and nonsmooth. To overcome this difficulty, we adopt the Accelerated Proximal Gradient (APG) method established by~\cite{li2015accelerated}, together with some linearization operators, to solved the raised optimization problem. In particular, we provide some theoretical analyses for the convergence of our optimization algorithm. Namely, the solution produced by the proposed algorithm is analytically proved to be a stationary point.
\end{itemize}

We conduct experiments on both synthetic and real datasets, and we also compare with some state-of-the-art methods. The results show that, in terms of recovery accuracy, our method is distinctly better than all competing methods.

\section{Related Work}
\subsection{Linear low-rank recovery}
Recently, linear low-rank recovery has attracted great attention due to its pleasing efficacy in exploring the low-dimensional structures from given measurements. Formally, the linear low-rank recovery problem can be directly or indirectly written in the following form:
\begin{equation}\label{eq:rpca_l}
	\begin{split}
		\mathop{\min}\limits_{\mathbf{A}} ~& {\|\mathbf{A}\|_{*}+ \lambda \|\mathbf{A}-\mathbf{X}\|_{\ell}},
	\end{split}
\end{equation}
where $\mathbf{X}$ and $\mathbf{A} \in \mathbb{R}^{d\times n}$ represent the given data and the desired structure, respectively. $\mathbf{A}-\mathbf{X}$ is the error residue. $\|\cdot\|_{\ell}$ is a certain robust norm to measure the residual between the observed and recovered signals. $\|\mathbf{\cdot}\|_{*}$ denotes a low-rank structure regularization and $\lambda$ is a non-negative parameter that provides a trade-off between the recovery fidelity and the low-rank promoting regularizer. The major difference among existing recovery methods is pertaining to the choice of penalty on the residual. Candès et al.~\cite{candes2011robust} choose $\ell_1$ norm to model the sparse noise. They theoretically prove that their model can \emph{exactly} recover the ground-truth data with the assumption of sparse outliers/noise. The works in~\cite{xu2010robust,zhang2015exact} select $\ell_{21}$ norm to penalize the column-sparse residual. Their model can also recover the correct column space of data.
The linear low-rank recovery has been applied to many computer vision tasks, such as face recognition~\cite{zheng2014fisher} and image classification~\cite{zhang2015image}, where they perform very well.
Besides, for low-rank matrix recovery, Liu et al.~\cite{liu2013fast} propose a fast tri-factorization method, and Cui et al.~\cite{cui2018exact} come up with a transformed affine matrix rank minimization method.
\subsection{Kernel low-rank method}
KPCA, an widespread extension of traditional PCA, seeks a low-rank approximation of the affinity among the data points in the kernel space~\cite{scholkopf1998nonlinear}. Similar to PCA, it is sensitive to the outliers even after mapping. Hence, some robust kernel low-rank methods have been proposed and investigated. In particular, the works in~\cite{baghshah2011learning,ji2017low,xie2018implicit} provide kernel low-rank methods for subspace clustering, which demonstrate that the kernel low-rank approximation does benefit the clustering of non-linear data.
Nguyen et al.~\cite{nguyen2015kernel} apply the kernel low-rank representation to face recognition.
Works in~\cite{pan2011learning,rakotomamonjy2014} investigate the influence of different kernels.
Garg et al.~\cite{garg2016non} present a new way to pursue the low-rankness in the kernel space, but the measurement of the other regularization is still in the original space, which cannot be directly utilized to solve Problem \ref{NLRPCA}.
\par
Though the existing methods have achieved great success for the clustering or linear low-rank recovery tasks, none of them can \emph{robustly} recover the non-linear or super low dimensional data in the original space. Comparatively, our model solves Problem \ref{NLRPCA} robustly when $\mathbf{A}$ is implicitly low-rank but could be high-rank or even full-rank by itself.

\section{Kernel Low-Rank Recovery}
\subsection{Problem Formulation}
The model, for solving the linear low-rank recovery problem with column-wise noise, can be represented as:
\begin{equation}\label{eq:rpca}
	\begin{split}
		\mathop{\min}\limits_{\mathbf{A}} ~& {\|\mathbf{A}\|_{*}+ \lambda \|\mathbf{A}-\mathbf{X}\|_{2,1}},
	\end{split}
\end{equation}
where $\|\cdot\|_*$ is the nuclear norm (sum of all singular values) and the $\ell_{21}$-norm can be calculated as $\|\mathbf{C}\|_{2,1} = \sum_i\|\mathbf{C}_{:,i}\|_2$. To tackle the issue of implicitly low-rank data, it is worthwhile to kernelize the model in~(\ref{eq:rpca}) to handle the data which are sampled from some complex nonlinear manifold. Moreover, in the scenario that the ambient dimension $d$ is far greater than the data size $n$, kernel method is more efficient.

Let $\phi: \mathbb{R}^d \rightarrow H$ be a mapping from the input space to the reproducing kernel Hilbert space $H$. Here we assume that ${\Phi(\mathbf{a}_i)}^n_{i=1}$ resides in a certain linear subspace in $H$. Namely, the non-linear observation is considered to be linearly dependent in $H$. Let $\mathbf{K} \in \mathbb{R}^{n\times n}$ be a positive semidefinite kernel Gram matrix whose elements are computed as:
\begin{equation}
	\begin{split}
		K_{ij} = \left(\phi(\mathbf{A})^T\phi(\mathbf{A})\right)_{ij} =  \phi(\textbf{a}_i)^T\phi(\textbf{a}_j) = ker(\textbf{a}_i,\textbf{a}_j),\nonumber
	\end{split}
\end{equation}
where $ker: \mathbb{R}^d \times \mathbb{R}^d \rightarrow R$ is the kernel function and
\begin{equation}
	\begin{split}
		\phi(\mathbf{A}) = [\phi(\textbf{a}_1),\phi(\textbf{a}_2),\cdots,\phi(\textbf{a}_n)].\nonumber
	\end{split}
\end{equation}
With the above assumption, by kernelizing model (\ref{eq:rpca}), our model can be represented as:
\begin{equation}\label{eq:generalform}
	\begin{split}
		\mathop{\min}\limits_{\mathbf{A}} ~& {\|\phi(\mathbf{A})\|_{*}+ \lambda \|\phi(\mathbf{A})-\phi(\mathbf{X})\|_{2,1}}.
	\end{split}
\end{equation}
Note that, after mapping, the data matrix still contains column-wise noise or outliers. Hence, we also adopt the $\ell_{21}$-norm in (\ref{eq:generalform}) to measure the error residue in the kernel space.

\subsection{Reformulation and Relaxation}
It is hard to optimize (\ref{eq:generalform}) due to the explicit dependency on $\phi(\mathbf{A})$. Fortunately, as shown in~\cite{garg2016non}, a symmetric and positive semi-definite matrix $\mathbf{K}$ can be factorized. We can easily derive the following proposition.

\begin{prop}
Assume $\mathbf{K}$ is a kernel Gram matrix which is computed as $\mathbf{K} =\phi(\mathbf{A})^T\phi(\mathbf{A})$, then we have
	\begin{equation}\label{prop1}
		\begin{split}
			\|\mathbf{B}\|_{*} = \|\phi(\mathbf{A})\|_{*}, ~~\forall ~~\mathbf{B}~:~\mathbf{K} = \mathbf{B}^T\mathbf{B},
		\end{split}
	\end{equation}
	where $\mathbf{B} \in \mathbb{R}^{n\times n}$.
\end{prop}

Substituting (\ref{prop1}) into (\ref{eq:generalform}), we convert (\ref{eq:generalform}) into:
\begin{equation}\label{Auxiliary}
	\begin{split}
		\mathop{\min}\limits_{\mathbf{A},\mathbf{B}} ~& {\|\mathbf{B}\|_{*}+ \lambda \|\phi(\mathbf{A})-\phi(\mathbf{X})\|_{2,1}},
		\\ \textbf{s.t.}  ~~&\mathbf{B}^T\mathbf{B} = \phi(\mathbf{A})^T\phi(\mathbf{A}).
	\end{split}
\end{equation}
We then relax the constrained problem to the following unconstrained one:
\begin{equation}\label{relax}
	\begin{split}
		\mathop{\min}\limits_{\mathbf{A},\mathbf{B}} ~&~ \|\mathbf{B}\|_{*}+ \lambda \|\phi(\mathbf{A})-\phi(\mathbf{X})\|_{2,1}+ \frac{\rho}{2}\|\mathbf{B}^T\mathbf{B} - \phi(\mathbf{A})^T\phi(\mathbf{A})\|_F^2,
	\end{split}
\end{equation}
where $\rho > 0$ is a parameter which balances the difference
$\mathbf{B}^T\mathbf{B} - \phi(\mathbf{A})^T\phi(\mathbf{A})$ and the original objective function. We can see that when $\rho$ is sufficiently large, (\ref{relax}) and (\ref{Auxiliary}) are the same model. It is worth mentioning that (\ref{Auxiliary}) can be solved by adopting the alternative direction method of multipliers (ADMM) technique. However, the optimization of the subproblem related to $\mathbf{A}$ is nonconvex and an auxiliary variable will be introduced. ADMM fails to ensure the convergence when the optimization involves more than three variables. Therefore, we choose an APG based method for our nonconvex and nonsmooth problem whose convergence can be guaranteed~\cite{li2015accelerated}. Another advantages of the relaxation is that sometimes the rank of ground-truth matrix $\phi(\mathbf{A})$ is higher than that of the solution of (\ref{Auxiliary}), which is caused by some unsuitable $\phi(\cdot)$. the solution of (\ref{relax}) is closer to the ground-truth in this case, and thus (\ref{relax}) is more robust to the selection of mapping functions.

\subsection{Optimization Algorithm}
We will show how to solve (\ref{relax}) in this subsection. We minimize the objective function alternately over $\mathbf{B}$ and $\mathbf{A}$. The updating of $\mathbf{A}$ is performed by the Monotone APG together with some linear approximation. Meanwhile, the subproblem involving $\mathbf{B}$ has a closed-form solution.\\
\textbf{(1) Update $\mathbf{B}$}

$\mathbf{B}$ can be updated by solving the following subproblem:
\begin{equation}\label{sub2}
	\begin{split}
		\mathop{\min}\limits_{\mathbf{B}}  {\|\mathbf{B}\|_{*}+ \frac{\rho}{2}\|\mathbf{B}^T\mathbf{B} - \mathbf{K_A}\|_F^2},
	\end{split}
\end{equation}
where $\mathbf{K_A} = \phi(\mathbf{A})^T\phi(\mathbf{A})$. Denote the singular value
decomposition~(SVD) of $\mathbf{K_A}$ as  $\mathbf{K_A} = \mathbf{U\Sigma V}^T$, and this subproblem has a closed-form solution given by \cite{garg2016non}:
\begin{equation}\label{sub2solution}
	\begin{split}
		\mathbf{B}^* =  \mathbf{\Gamma}^*\mathbf{V}^T.
	\end{split}
\end{equation}
$\mathbf{\Gamma}^*$ is a diagonal matrix with $\Gamma_{ii} = \argmin_{\gamma>0} \frac{\rho}{2}(\sigma_i - \gamma^2)^2 + \gamma$, where $\sigma_i$ is
the $i$-th singular value of $\mathbf{K_A}$. Hence, each $\Gamma_{ii}$ can be achieved by solving a cubic equation. Note that $\mathbf{B}^*$ is not unique since one can multiply an arbitrary unitary matrix to the left of (\ref{sub2solution}) without changing the objective value in (\ref{sub2}). Fortunately, the  non-uniqueness does not affect the optimization of $\mathbf{A}$ since only $(\mathbf{B}^{*})^T\mathbf{B}^*$ involves the updating of $\mathbf{A}$.\\
\textbf{(2) Update $\mathbf{A}$}

To update $\mathbf{A}$, the following subproblem should be solved:
\begin{equation}\label{sub1}
	\begin{split}
		\mathop{\min}\limits_{\mathbf{A}} { \|\phi(\mathbf{A})-\phi(\mathbf{X})\|_{2,1}+ \frac{\alpha}{2}\|\mathbf{B}^T\mathbf{B} - \phi(\mathbf{A})^T\phi(\mathbf{A})\|_F^2},
	\end{split}
\end{equation}
where $\alpha = \rho/\lambda$.  By dividing the matrix $\*A$ into columns, (\ref{sub1}) can be rewritten as:
\begin{equation}\label{sub1decomp}
	\begin{split}
		\mathop{\min}\limits_{\mathbf{a}_1,\cdots,\mathbf{a}_n} \sum_{i=1}^{n}\left\{\|\phi(\mathbf{a}_i)-\phi(\mathbf{x}_i)\|_{2}+ \frac{\alpha}{2}\|\mathbf{m}_i - \phi(\mathbf{A})^T\phi(\mathbf{a}_i)\|_2^2\right\},\nonumber
	\end{split}
\end{equation}
where $\mathbf{m}_i$ is the $i$-th column of $\mathbf{B}^T\mathbf{B} $.  The solution of this problem can be achieved by the block coordinate descent
(BCD) method \cite{xu2013block} which minimizes the objective cyclically over each of $\mathbf{a}_1,\cdots, \mathbf{a}_n$ while fixing the remaining blocks at their last updated values. Hence, we are required to address the following problem:
\begin{equation}\label{sub1decomp2}
	\begin{split}
		\mathop{\min}\limits_{\mathbf{a}_i} ~& \sqrt{\phi(\*{a}_i)^T\phi(\*{a}_i) + \phi(\*{x}_i)^T\phi(\*{x}_i) - 2\phi(\*{x}_i)^T\phi(\*{a}_i)}+ \frac{\alpha}{2}\sum_{j=1}^{n}\left(m_{ij} - \phi(\*{a}_i)^T\phi(\*{a}_j)\right)^2.
	\end{split}
\end{equation}
To optimize this problem, it requires to define the kernel function $ker: \mathbb{R}^d \times \mathbb{R}^d \rightarrow R$. Here we choose two types of kernels (convex and non-convex) as the examples. The optimization related to other kernel functions can be solved in a similar way.

\textbf{(i) Convex kernel:} We select the most commonly used convex kernel, i.e., Polynomial Kernel Function ($d \geq 1$). The inner product in the kernel space can be represented as
$$\phi(\mathbf{a}_i)^T\phi(\mathbf{a}_j) ~=~ (\*a_i^T\*a_j + c)^d$$
where $c \geq 0$ is a free parameter trading off the influence of higher-order versus lower-order terms in the polynomial. $d$ is the order of the polynomial kernel. (\ref{sub1decomp2}) can be rewritten as
\begin{equation}\label{sub1poly}
	\begin{split}
		\mathop{\min}\limits_{\mathbf{a}_i} ~& \sqrt{(\*a_i^T\*a_i + c)^d + (\*x_i^T\*x_i + c)^d - 2(\*a_i^T\*x_i + c)^d}+ \frac{\alpha}{2}\sum_{j=1}^{n}\left(m_{ij} - (\*a_i^T\*a_j + c)^d)\right)^2 .
	\end{split}
\end{equation}
Note that, $\sqrt{(\cdot)}$ is a real-valued function and it is differentiable at non-zero points. Thus we utilize its linear approximation at point $\mathbf{a}_i - \mathbf{x}_i$ to simplify and accelerate the optimization.
\begin{equation}\label{sub1expHQ}
	\begin{split}
		\mathop{\min}\limits_{\mathbf{a}_i} ~&~ \frac{\alpha}{2}\sum_{j=1}^{n}\left(m_{ij} - (\*a_i^T\*a_j + c)^d\right)^2 + 2d\delta_i(\tau_{a_i}\*a_i - \tau_{x_i}\*x_i)^T(\*a_i - \*x_i),
	\end{split}
\end{equation}
where $\delta_i = 1/\sqrt{(\*a_i^T\*a_i + c)^d + (\*x_i^T\*x_i + c)^d - 2(\*a_i^T\*x_i + c)^d + \mu^2}$, $\tau_{a_i} = (\*a_i^T\*a_i + c)^{d-1}$, $\tau_{x_i} = (\*a_i^T\*x_i + c)^{d-1}$, and $\mu>0$ is the smooth parameter.
Obviously, one local minimizer $\mathbf{a}_i$ can be calculated in an alternating minimization way:
\begin{align}
	&~\tau_{a_i}^{k+1} = \left((\*a_i^k)^T\*a_i^k + c\right)^{d-1},~~\tau_{x_i}^{k+1} = \left(\*x_i^T\*a_i^k + c\right)^{d-1},\label{sub1ployHQUpdate1}\\
	&~\delta_i^{k+1} = 1/\sqrt{((\*a_i^k)^T\*a_i^k + c)^d + (\*x_i^T\*x_i + c)^d - 2((\*a_i^k)^T\*x_i + c)^d + \mu^2},\label{sub1ployHQUpdate2}\\
	&~\mathbf{a}_i^{k+1} = \argmin_{\mathbf{a}_i}~  \mu_i^{k+1}\sum_{j=1}^{n}\left(m_{ij} - (\*a_i^T\*a_j + c)^d\right)^2 + (\tau_{a_i}^{k+1}\*a_i - \tau_{x_i}^{k+1}\*x_i)^T(\*a_i - \*x_i).\label{sub1ployHQUpdate3}
\end{align}
where $\mu_i^{k+1} = \alpha / (4d\delta_i^{k+1})$.

\textbf{(ii) Non-Convex kernel:}
For non-convex kernel, we choose the Gaussian Kernel Function for mapping the observation into an infinite-dimensional space. The inner product in the kernel space can be represented as $\phi(\mathbf{a}_i)^T\phi(\mathbf{a}_j) ~=~ \exp(-\gamma \|\mathbf{a}_i- \mathbf{a}_j\|_2^2)$, where $\mathbf{a}_i \in \mathbb{R}^d$ and $\gamma >0$ is the precision parameter of the Gaussian Kernel Function. (\ref{sub1decomp2}) can be rewritten as:
\begin{equation}\label{sub1exp}
	\begin{split}
		\mathop{\min}\limits_{\mathbf{a}_i} ~& \sqrt{2 - 2\exp(-\gamma \|\mathbf{a}_i - \mathbf{x}_i\|_2^2)}+ \frac{\alpha}{2}\sum_{j=1}^{n}\left(m_{ij} - \exp(-\gamma\|\mathbf{a}_i-\mathbf{a}_j\|_2^2)\right)^2 .
	\end{split}
\end{equation}
Note that, $\sqrt{\exp(\cdot)}$ is a real-valued function and it is differentiable at non-zero points. Thus we utilize its linear approximation at point $\mathbf{a}_i - \mathbf{x}_i$ to simplify and accelerate the optimization.
The problem of (\ref{sub1exp}) is converted  into:
\begin{equation}\label{sub1expHQ}
	\begin{split}
		\mathop{\min}\limits_{\mathbf{a}_i} ~&~ \frac{\alpha}{2}\sum_{j=1}^{n}\left(m_{ij} - \exp(-\gamma\|\mathbf{a}_i-\mathbf{a}_j\|_2^2)\right)^2 +2\beta_i p_i\gamma \|\mathbf{a}_i - \mathbf{x}_i\|_2^2,
	\end{split}
\end{equation}
where $\beta_i = 1/\sqrt{2 - 2\exp(-\gamma \|\mathbf{a}_i - \mathbf{x}_i\|_2^2) + \mu^2}$, $p_i = \exp(-\gamma \|\mathbf{a}_i - \mathbf{x}_i\|_2^2)$, and $\mu>0$ is the smooth parameter.
Obviously, one local minimizer $\mathbf{a}_i$ can be calculated in an alternating minimization way:
\begin{align}
	&~p_i^{k+1} =  \exp(-\gamma  \|\mathbf{a}_i^{k} - \mathbf{x}_i\|_2^2),\label{sub1expHQUpdate1}\\
	&~\beta_i^{k+1} = 1/\sqrt{2 - 2\exp(-\gamma \|\mathbf{a}_i^{k} - \mathbf{x}_i\|_2^2) + \mu^2},\label{sub1expHQUpdate2}\\
	&~\mathbf{a}_i^{k+1} = \argmin_{\mathbf{a}_i}~  \rho_i^{k+1}\sum_{j=1}^{n}\left(m_{ij} - \exp(-\gamma\|\mathbf{a}_i-\mathbf{a}_j\|_2^2)\right)^2 +\|\mathbf{a}_i - \mathbf{x}_i\|_2^2, \label{sub1expHQUpdate3}
\end{align}
where $\rho_i^{k+1} = \alpha / (4\gamma\beta_i^{k+1} p_i^{k+1})$.
Note that, in most cases, the solution to the linear approximation problem is not exactly equivalent to that of the original problem. However, in contrast, here the updating steps (\ref{sub1expHQUpdate1})~--~(\ref{sub1expHQUpdate3}) can solve the optimization in (\ref{sub1exp}), which we will show in the next section.\\
\textbf{(3) Solve Nonconvex Programming}

The optimization problem in (\ref{sub1ployHQUpdate3}) or (\ref{sub1expHQUpdate3}) is a nonconvex programming whose solution can be attained  by the APG method. The updating steps of $\mathbf{a}_i$  includes:
\begin{align}
	&~\mathbf{y}_i^{k}=  \mathbf{a}_i^{k} +\frac{t^{k-1}}{t^{k}}(\mathbf{z}^k - \mathbf{a}_i^k) +\frac{t^{k-1}-1}{t^{k}}(\mathbf{a}_i^k - \mathbf{a}_i^{k-1}),\label{sub1expAPGUpdate1}\\
	&~\mathbf{z}_i^{k+1} = \text{prox}_{\delta g}\left(\mathbf{y}^{k} - \delta\nabla f(\mathbf{y}^{k})\right),\label{sub1expAPGUpdate2}\\
	&~\mathbf{v}_i^{k+1} = \text{prox}_{\delta g}\left(\mathbf{a}_i^{k} - \delta\nabla f(\mathbf{a}_i^{k})\right),\label{sub1expAPGUpdate3}\\
	&~t^{k+1} = \frac{\sqrt{4(t^k)^2+1}+1}{2},\\
	&~\mathbf{a}_i^{k+1} =\left\{
	\begin{aligned}
		&~\mathbf{z}^{k+1},~~~\text{if}~F(\mathbf{z}_i^{k+1})\leq F(\mathbf{v}_i^{k+1}),\\
		&~\mathbf{v}^{k+1},~~~\text{otherwise}. \\
	\end{aligned}
	\right.\label{sub1expAPGUpdate4}
\end{align}
where $f(\*a_i)$ is $ \mu_i^{k+1}\sum_{j=1}^{n}\left(m_{ij} - (\*a_i^T\*a_j + c)^d\right)^2$ for Polynomial Kernel or $\rho_i^{k+1}\sum_{j=1}^{n}\left(m_{ij} - \exp(-\gamma\|\*a_i-\mathbf{a}_j\|_2^2)\right)^2$ for Gaussian Kernel. $\nabla f(\cdot)$ is the gradient of $f(\cdot)$ and $g(\*a_i)$ represents $(\tau_{a_i}^{k+1}\*a_i - \tau_{x_i}^{k+1}\*x_i)^T(\*a_i - \*x_i)$  or $\|\*a_i - \mathbf{x}_i\|_2^2$. The proximal mapping is defined as $\text{prox}_{\delta g}(\mathbf{x}) = \argmin_u g(\mathbf{u}) + \frac{1}{2\delta}\|\mathbf{x}-\mathbf{u}\|_2^2$. $\delta$ is a fixed constant satisfying $\delta < 1/L$. $L$ is the Lipschitz constant of $\nabla f(\cdot)$ and $F(\cdot)$ denotes $f(\cdot) + g(\cdot)$.

The algorithm to solve (\ref{relax}) with the APG and alternating minimization is outlined in Algorithm 1.
\begin{table}[t]
	\renewcommand\arraystretch{1.0}
	\normalsize
	\setlength{\tabcolsep}{0.45em}
	\begin{tabular*}{\linewidth}{l}
		\hline
		\textbf{Algorithm 1}  ~~Proposed Algorithm for solving (\ref{relax}) \\
		\hline
		\parbox{\linewidth}{
			\textbf{Input}: data matrix $\mathbf{X}$, parameter $\lambda,\rho>0$, kernel Gram matrix $\mathbf{K}$ for initialization (optional).
		} \\
		
		\parbox{\linewidth}{
			\textbf{Output}: recovered non-linear data matrix $\mathbf{A}^{*}$.
		} \\
		
		\parbox{\linewidth}{
			\textbf{Initialize}: $\mathbf{A} = 0$, $\mathbf{B}^T\mathbf{B}= \mathbf{K}~~\text{or}~~\phi(\mathbf{X})^T\phi(\mathbf{X})$.
		} \\
		
		\parbox{\linewidth}{
			\textbf{For} $k = 1,2,3,\dots $\textbf{do}
		} \\

		\parbox{\linewidth}{
			~~\textbf{1:} Update $\mathbf{A}^{k+1}$:
		} \\
		
		\parbox{\linewidth}{
			~~~~~~~~~~~\textbf{For} $i = 1,2,3,\dots,N$ \textbf{do}
		} \\
		
		\parbox{\linewidth}{
			~~~~~~~~~~~~~~\textbf{(1):} Update $p_i^{k+1}$ and $\beta_i^{k+1}$ by (\ref{sub1expHQUpdate1})~--~(\ref{sub1expHQUpdate2}).
		} \\
		
		\parbox{\linewidth}{
			~~~~~~~~~~~~~~\textbf{(2):} Update $\mathbf{a}_i^{k+1}$ and $\mathbf{v}_i^{k+1}$ by (\ref{sub1expAPGUpdate1})~--~(\ref{sub1expAPGUpdate4}).
		} \\
		
		\parbox{\linewidth}{
			~~\textbf{2:} Update $\mathbf{B}^{k+1}$ by (\ref{sub2solution}).
		} \\
		
		\parbox{\linewidth}{
			~~\textbf{3:} $k = k + 1$;
		} \\
		
		\parbox{\linewidth}{
			\textbf{end For}
		} \\
		
		\hline
	\end{tabular*}
\end{table}

\subsection{Computational Complexity}
The updating of $\mathbf{B}$ consists of two parts, finding roots for $n$ cubic equations and performing the SVD operator on $\phi(\mathbf{A})^T\phi(\mathbf{A})$. The computational complexity for achieving the roots is $O(n)$, since we can get the closed-form expression for the roots of cubic equations. The complexity for the SVD is $O(rn^2)$, where $n$ is the size of data and $r$ is the rank of $\phi(\mathbf{A})^T\phi(\mathbf{A})$. During the procedure of updating $\mathbf{a}_i$ according to (\ref{sub1expAPGUpdate1}) - (\ref{sub1expAPGUpdate3}), the $O(n^2)$ matrix vector multiplication needs to be carried out. Hence, the computational complexity for calculating $\mathbf{A}$ is $O(n^3)$. In summary, the total computational complexity for the whole algorithm in each iteration is $O(n^3 + rn^2)$.

\section{Theoretical Analysis}
In this section, we first provide some useful theoretical results, including Lemma \ref{lem1} for illustrating the connection between (\ref{sub1exp}) and (\ref{sub1expHQ}), as well as Theorem \ref{thm1} and Theorem \ref{thm2} for ensuing the convergence of the optimization.

Before stating the Lemma 1, we first introduce one proposition to rewrite the non-linear mapping $\exp(\cdot)$ by its conjugated function. Based on the theory of convex conjugated functions~\cite{rockafellar2015convex}, we can derive the following proposition.
\begin{prop}\label{HQProp}
	There exists a convex conjugated function $\varphi$ of $\exp(\cdot)$ such that
	\begin{equation}\label{HQprop}
		\begin{split}
			\exp(\gamma \|\mathbf{x}\|_2^2) = - \mathop{\min}\limits_{p}  \left(~p\gamma \|\mathbf{x}\|_2^2 - \varphi(p)~\right),
		\end{split}
	\end{equation}
	where $p \in R$ is a scalar variable. For a fixed $\mathbf{x}$, the minimum is reached at $p^* = -\exp(\gamma \|\mathbf{x}\|_2^2)$~\cite{he2011robust}.
\end{prop}
Base on the above proposition, we have the following connections between (\ref{sub1exp}) and (\ref{sub1expHQ}):

\begin{lem}\label{lem1}
	Cyclic iteration between steps (\ref{sub1expHQUpdate1})~--~(\ref{sub1expHQUpdate3}) can solve the optimization in (\ref{sub1exp}).
\end{lem}
\begin{proof}
	We represent $\frac{\alpha}{2}\sum_{j=1}^{n}\left(m_{ij} - \exp(-\gamma\|\mathbf{a}_i-\mathbf{a}_j\|_2^2)\right)^2$ as $m(\mathbf{a}_i)$. With the same spirit of the iteratively reweighted least squares (IRLS) method~\cite{fornasier2011low}, we can solve (\ref{sub1exp}) by iteratively optimizing the following problem with the weight $\beta_i$ determined from the last iteration:
	\begin{equation}\label{aiargmin}
		\begin{split}
			\mathbf{a}_i^{k+1} = \mathop{\argmin}\limits_{\mathbf{a}_i} - 2\beta_i^k \exp(-\gamma \|\mathbf{a}_i - \mathbf{x}_i\|_2^2)+  m(\mathbf{a}_i),
		\end{split}
	\end{equation}
	where $\beta_i^k = 1/\sqrt{2 - 2\exp(-\gamma \|\mathbf{a}_i^k - \mathbf{x}_i\|_2^2) + \mu^2}$ and $\mu>0$ is the smooth parameter. Substituting (\ref{HQprop}) into (\ref{aiargmin}), it gives that:
	\begin{equation}
		\begin{split}
			\{\mathbf{a}_i^{k+1},p_i^{k+1}\} = \mathop{\argmin}\limits_{\mathbf{a}_i, p_i}  2\beta_i p_i \gamma \|\mathbf{a}_i - \mathbf{x}_i\|_2^2+  m(\mathbf{a}_i)  - \varphi(p).\nonumber
		\end{split}
	\end{equation}
	Proposition \ref{HQProp} gives that $p_i^{k+1} = \exp(-\gamma  \|\mathbf{a}_i^{k} - \mathbf{x}_i\|_2^2)$. Hence, we get:
	\begin{equation}\label{aiargmin3}
		\begin{split}
			\mathbf{a}_i^{k+1} ~=&~ \mathop{\argmin}\limits_{\mathbf{a}_i}  2\beta_i^{k} p_i^{k+1} \gamma \|\mathbf{a}_i - \mathbf{x}_i\|_2^2+  m(\mathbf{a}_i).
		\end{split}
	\end{equation}
	Due to (\ref{aiargmin}) --  (\ref{aiargmin3}), we find that steps (\ref{sub1expHQUpdate1})~--~(\ref{sub1expHQUpdate3}) actually solve the problem in (\ref{sub1exp}) by the iteratively reweighted strategy, and hence cyclic iteration between these steps can solve the optimization in (\ref{sub1exp}).
\end{proof}

We denote the objective of (\ref{relax}) as $F(\mathbf{A},\mathbf{B})$. Then the following theorem regarding the convergence of Algorithm 1 can be established.

\begin{thm}\label{thm1}
	The sequence $\{\mathbf{A}^k,\mathbf{B}^k, \mathbf{V}^k\}$ generated in Algorithm 1 satisfies the following properties:\\
	$(1)$ The objective $F(\mathbf{A}^k,\mathbf{B}^k)$ is monotonically decreasing, i.e.
	\begin{equation}\label{decreasing}
		\begin{split}
			 F(\mathbf{A}^{k},\mathbf{B}^{k}) - F(\mathbf{A}^{k+1},\mathbf{B}^{k+1}) \geq ~ \frac{\rho}{2}\|\mathbf{B}^{k+1}- \mathbf{B}^{k}\|_F^2 + \left( \frac{1}{2\delta} - \frac{L}{2}\right)\|\mathbf{V}^{k+1}-\mathbf{A}^k\|_F^2;
		\end{split}
	\end{equation}
	$(2)$ $\mathbf{B}^{k+1}-\mathbf{B}^k\rightarrow 0$, $\mathbf{V}^{k+1}-\mathbf{A}^k\rightarrow 0;$\\
	$(3)$ The sequence $\{\mathbf{A}^k\}$, $\{\mathbf{B}^k\}$ and $\{\mathbf{V}^k\}$ are bounded.
\end{thm}
\begin{proof}
	First, from the updating rule of $\mathbf{B}^{k+1}$ in (\ref{sub2solution}), we have
	\begin{equation}
		\begin{split}
			\mathbf{B}^{k+1}~=~\mathop{\argmin}\limits_{\mathbf{B}} F(\mathbf{A}^{k},\mathbf{B}).\nonumber
		\end{split}
	\end{equation}
	Note that $F(\mathbf{A}^{k},\mathbf{B})$ is $\rho$-strongly convex. By the Lemma B.5 in~\cite{mairal2013optimization}.We have
	\begin{equation}\label{FBk}
		\begin{split}
			~ F(\mathbf{A}^{k},\mathbf{B}^{k}) - F(\mathbf{A}^{k},\mathbf{B}^{k+1})\geq ~ \frac{\rho}{2}\|\mathbf{B}^{k+1}- \mathbf{B}^{k}\|_F^2.
		\end{split}
	\end{equation}
	Second, we denote the objective in (\ref{sub1expHQ}) as $f(\mathbf{a}_i,\mathbf{B})$, from the Theorem 1 in~\cite{li2015accelerated}, for all $i$, we have
	\begin{equation}\label{fak}
		\begin{split}
			~ f(\mathbf{a}_i^{k},\mathbf{B}_i^{k+1}) - f(\mathbf{a}_i^{k+1},\mathbf{B}_i^{k+1})\geq ~ \zeta\|\mathbf{v}_i^{k+1}-\mathbf{a}_i^k\|_2^2,
		\end{split}
	\end{equation}
	where $\zeta = ( \frac{1}{2\delta} - \frac{L}{2})$. As aforementioned, $\sum_{i=1}^n(f(\mathbf{a}_i,\mathbf{B}) + \beta_i)$ is the linear approximation of $F(\mathbf{A},\mathbf{B})$ at $\mathbf{A}^k$, which gives $\sum_{i=1}^n(f(\mathbf{a}_i^k,\mathbf{B}^{k+1}) + \beta_i^k) = F(\mathbf{A}^k,\mathbf{B}^{k+1})$. From the concavity of $-\exp(\cdot)$, we have $\sum_{i=1}^n(f(\mathbf{a}_i^{k+1},\mathbf{B}^{k+1}) + \beta_i^{k+1}) \geq F(\mathbf{A}^{k+1},\mathbf{B}^{k+1})$. Sum the inequality in (\ref{fak}) for all $i$, we get
	\begin{equation}
		\begin{split}
			F(\mathbf{A}^k,\mathbf{B}^{k+1}) - F(\mathbf{A}^{k+1},\mathbf{B}^{k+1})\geq ~ \zeta\|\mathbf{V}^{k+1}-\mathbf{A}^k\|_F^2.\nonumber
		\end{split}
	\end{equation}
	Thus, together with (\ref{FBk}), we achieve the conclusion in (\ref{decreasing}). Hence, $F(\mathbf{A}^k,\mathbf{B}^{k})$ is monotonically decreasing and thus it is upper bounded. This implies that $\{\mathbf{A}^k,\mathbf{B}^k\}$ is bounded.
	
	Now, summing (\ref{decreasing}) over $k = 0, 1, \cdots$, we have
	\begin{equation}
		\begin{split}
			\sum_{k=0}^\infty\frac{\rho}{2}\|\mathbf{B}^{k+1}- \mathbf{B}^{k}\|_F^2 + \zeta\|\mathbf{V}^{k+1}-\mathbf{A}^k\|_F^2 \leq F(\mathbf{A}^0,\mathbf{B}^{0}).\nonumber
		\end{split}
	\end{equation}
	This implies $\mathbf{B}^{k+1}-\mathbf{B}^k\rightarrow 0$ and $\mathbf{V}^{k+1}-\mathbf{A}^k\rightarrow 0$. Then, similar to $\{\mathbf{A}^k\}$, $\{\mathbf{V}^k\}$ is also bounded.\\
	The proof is completed.
\end{proof}

\begin{thm}\label{thm2}
	The sequence $\{\mathbf{A}^k,\mathbf{B}^k\}$ generated in Algorithm 1 has at least one accumulation point. Let $(\mathbf{A}^*,\mathbf{B}^*)$ be any accumulation point of $\{\mathbf{A}^k,\mathbf{B}^k\}$, and we have $0\in \partial F(\mathbf{A}^*,\mathbf{B}^*)$, i.e., $(\mathbf{A}^*,\mathbf{B}^*)$ is a stationary point.
\end{thm}
\begin{proof}
	Now, from the boundedness of $\{\mathbf{A}^k,\mathbf{B}^k\}$, there exist a point $(\mathbf{A}^*,\mathbf{B}^*)$ and a subsequence  $\{\mathbf{A}^{k_j},\mathbf{B}^{k_j}\}$ such that $\mathbf{A}^{k_j}\rightarrow \mathbf{A}^*$, $\mathbf{B}^{k_j}\rightarrow \mathbf{B}^*$. Then by (2) in Theorem 1, we have $\mathbf{V}^{k_j+1}\rightarrow \mathbf{A}^*$, $\mathbf{B}^{k_j+1}\rightarrow \mathbf{B}^*$. On the other hand, from the optimality of $\mathbf{B}^{k_j+1}$ to (\ref{sub2}), $\mathbf{V}^{k_j+1}$ to (\ref{sub1expAPGUpdate3}) and Theorem 1 in~\cite{li2015accelerated}, we have
	\begin{equation}
		\begin{split}
			0 \in \partial_\mathbf{B} F(\mathbf{A}^{k_j},\mathbf{B}^{k_j+1}),~~~ 0 \in \partial_\mathbf{A} F(\mathbf{V}^{k_j+1},\mathbf{B}^{k_j+1}).\\
		\end{split}\nonumber
	\end{equation}
	Let $k \rightarrow +\infty$ above. We have
	\begin{equation}
		\begin{split}
			0 \in \partial_\mathbf{B} F(\mathbf{A}^{*},\mathbf{B}^{*}),~~~ 0 \in \partial_\mathbf{A} F(\mathbf{A}^{*},\mathbf{B}^{*}).\\
		\end{split}\nonumber
	\end{equation}
	Hence, $(\mathbf{A}^*,\mathbf{B}^*)$ is a stationary point of (\ref{relax}).
\end{proof}

\section{Experimental Verification}
\subsection{Experimental Settings}
In this section, we conduct experiments on both synthetic and real datasets to show the advantages of our proposed method.

\textbf{Data}: The real datasets cover two computer vision tasks: 1) non-linear data recovery from the similarity; 2) non-linear data denoising over the MNIST~\cite{salakhutdinov2008quantitative} and COIL-20~\cite{nene1996columbia} databases. The MNIST database consists of $8$-bit grayscale images of digits from "$0$" to "$9$". Each image is centered on a $28 \times 28$ grid. The COIL-20 database contains 1440 samples distributed over 20 objects, where each image is with the size of $32 \times 32$.

\textbf{Baselines}: We assess the performance of the proposed model in comparison with several state-of-the-art methods including Outlier Pursuit (OP)~\cite{xu2010robust}, KPCA~\cite{nguyen2009robust} and GRPCA~\cite{shahid2015robust}, the codes of which are downloaded from the authors’ websites except KPCA. We implement KPCA according to the paper. All methods' settings follow the suggestions by the authors or the given parameters.

\textbf{Evaluation metrics}: Two metrics are used to evaluate the performance of data recovery methods.\\
\emph{-- Peak Signal-to-Noise Ratio (PSNR) }: Suppose Mean Squared Error (MSE) is defined as $\|\mathbf{X}_0-\mathbf{X}_{rev}\|_F^2/dn$, where $\mathbf{X}_0, \mathbf{X}_{rev} \in \mathbb{R}^{n\times d}$ are the original image and the recovered image, respectively, then the PSNR value can be calculated by $10\log_{10}(255^2/\text{MSE})$.\\
\emph{-- Signal-to-Noise Ratio (SNR) }: The SNR can be calculated by $10\log_{10}(\|\mathbf{X}_{0}\|_F^2/\|\mathbf{X}_0-\mathbf{X}_{rev}\|_F^2)$.

\subsection{Data Recovery with Graph Constraint}
\begin{figure}[ht]
	\centering
	\includegraphics[width=80mm]{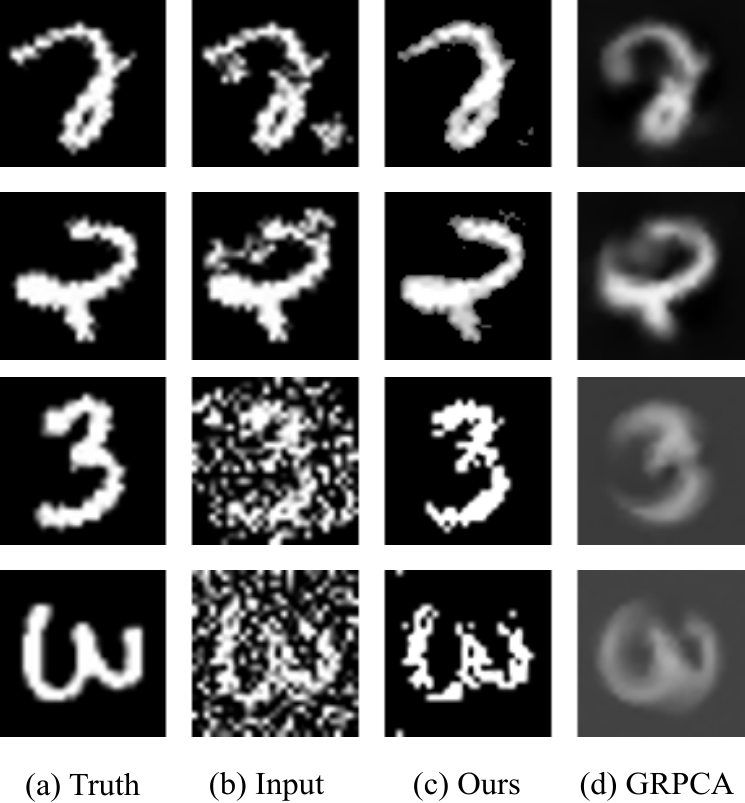}
	\caption{(a) is the original data. (b) is the corrupted observation with Gaussian or structure noise. (c) is the result provided by our non-linear recovery method. (d) is the recovery results of GRPCA. }\label{recovSimi}
\end{figure}
In this experiment, we aim at recovering the data with a graph constraint. For our proposed model, we solve the subproblem (\ref{sub1}) with $\mathbf{B}^T\mathbf{B}$ fixed. Note that, except GRPCA, all other comparative methods cannot cope with this similarity recovery task. We examine the effectiveness of our model over the MNIST database. Firstly, we randomly select $400$ images from digit "$2$" and "$3$", and then rotate them with a random degree $\theta \in [-90^\circ, 90^\circ]$. Secondly, $10\%$ of the images are randomly chosen to corrupt: for each chosen image $\mathbf{x}$, its observation is computed by adding Gaussian noise with zero mean and standard deviation $0.5\|\mathbf{x}\|_F^2$, or adding three blocks of structure occlusion with the size of $4\times 4$. Finally, we convert these images to vectors of $784$ dimension. In order to construct the graph constraint for our proposed model and GRPCA, we adopt the same way as in~\cite{shahid2015robust} and the input graph is calculated from the $5$-nearest neighbors. Note that we utilize the Gaussian kernel function with $\gamma = 0.007$ on digit "$2$" and Polynomial Kernel Function with $d = 2$ on digit "$3$" .

Fig.~\ref{recovSimi} shows the results of our method and GRPCA on the rotated MNIST data. As we can see, the proposed method produces the encouraging recovery results and outperforms the competing method. This confirms the superiority of our model in the setting of highly non-linear scenario. It is worthwhile to mention that the method used to solve the problem in (\ref{sub1}) can be directly applied into some other scenarios, such as multi-modal inference and multi-view learning for recovery from similarity.
\subsection{Data Denoising}

\begin{figure}[!t]
	\centering
	\includegraphics[width=100mm]{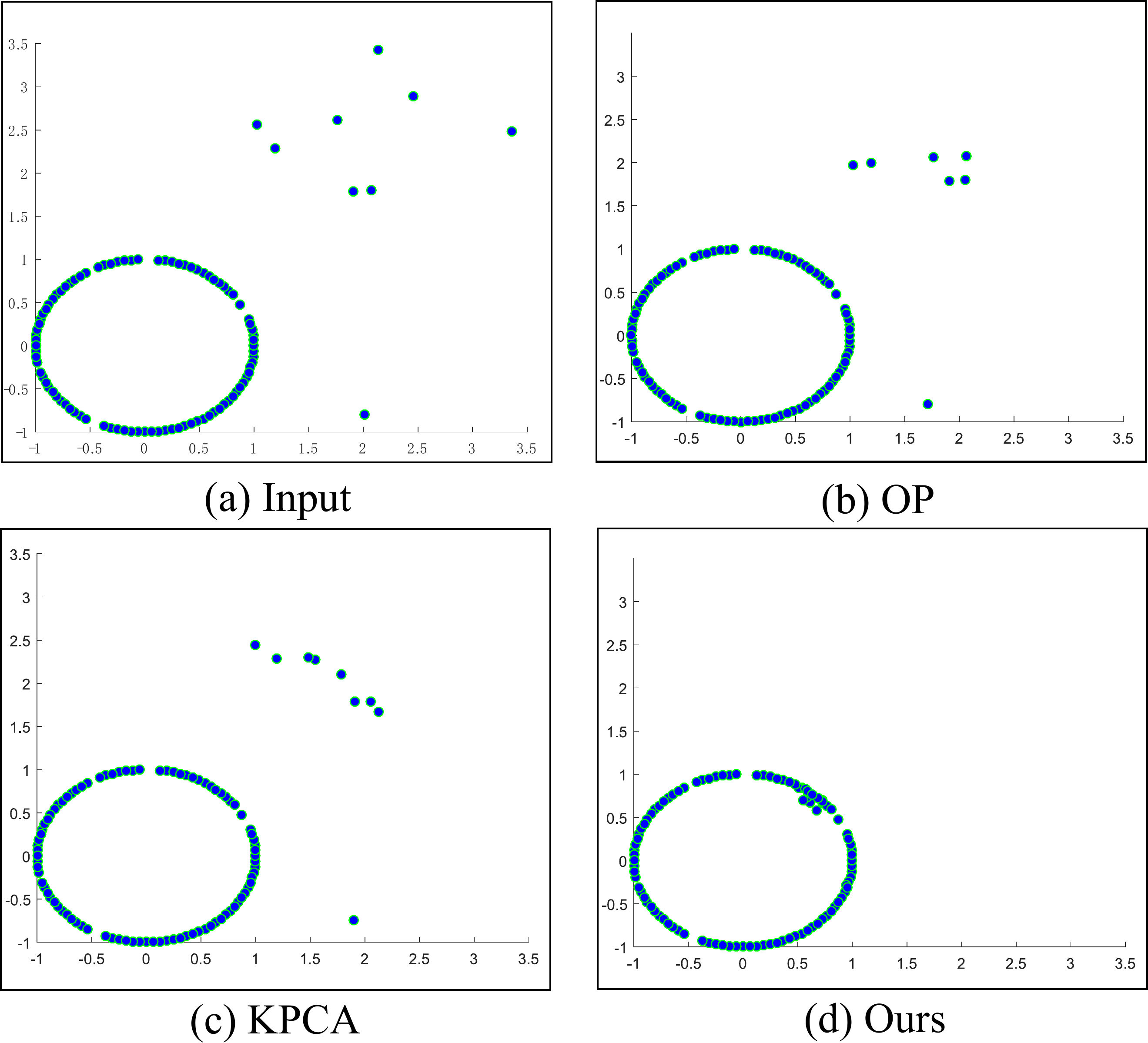}
	\caption{(a) The input data. (b) The recovery results of OP. (c) The recovery results of KPCA. (d) is the result of our kernel low-rank recovery method.}\label{toydata}
\end{figure}

\begin{figure}[!t]
	\centering
	\includegraphics[width=100mm]{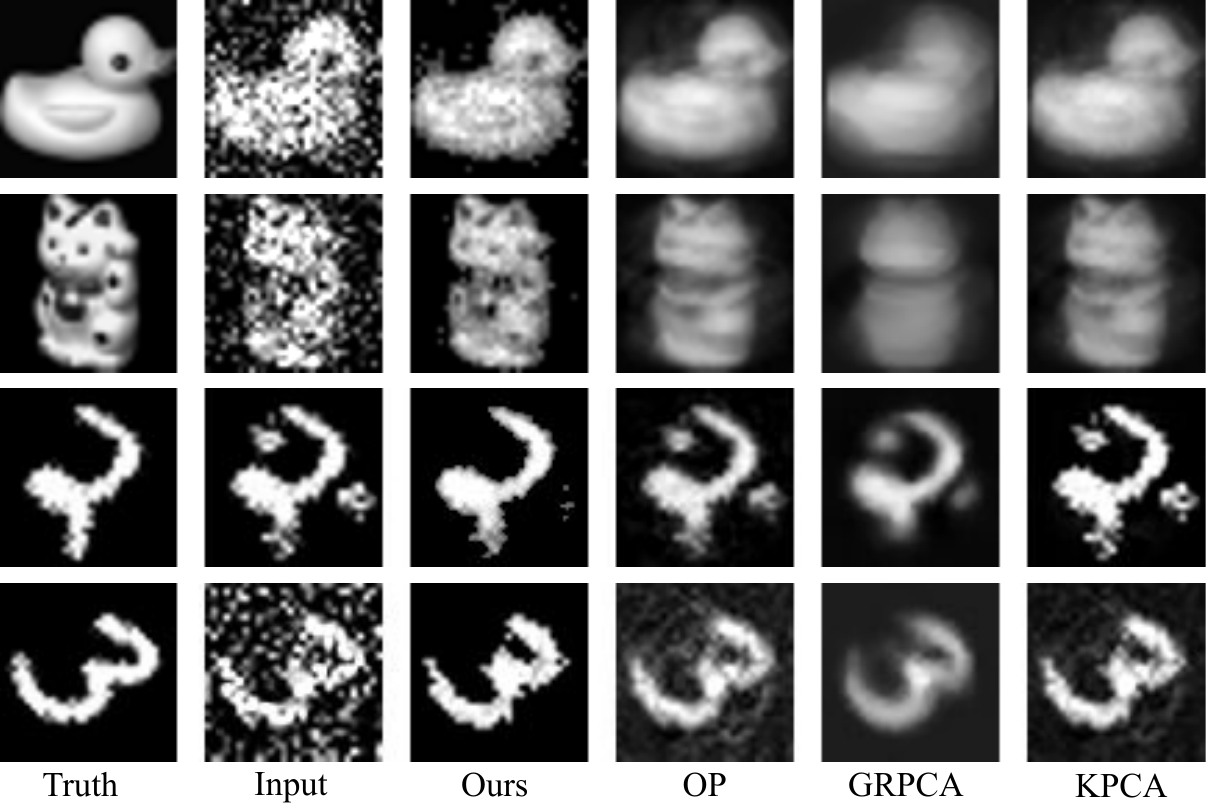}
	\caption{ Visual comparison on the task of Data Denoising.}\label{denoised}
\end{figure}

We now evaluate the effectiveness of our method on the data denoising problem.

1)~\emph{Two-Dimensional Case}: Fig.~\ref{toydata} shows the results of methods on some synthetic data. We randomly select 100 data points from a circle embedded in a two-dimensional plane, resulting in a $2\times100$ clean data matrix. We then select 10\% data points as the outliers. In this example, the ambient data dimension is equal to the extrinsic dimension, and thus traditional low-rankness based methods (e.g., OP) cannot recover the data points correctly. In sharp contrast, as shown in Fig.~\ref{toydata}, our method can still identify the outliers and replace them by the points which are close to the ground-truth manifold.

2)~\emph{High-Dimensional Case}: We apply the proposed method to denoise data over the MNIST and the COIL-20 databases. We compare all the recovery methods in two cases: (1) rotation with Gaussian noise, (2) rotation with occlusion. For the MNIST database, we randomly select $400$ images from digit "$2$" and "$3$", and then rotate them with a random degree. For the COIL-20 database, we randomly choose $10$ subjects and their corresponding $72$ images, and then rotate each image $5$ times with a degree from $-90^\circ$ to $90^\circ$. In these two cases, $10\%$ data is randomly chosen to corrupt in the same way as the previous experiment. Finally, we convert these images to vectors and normalize them to a unit length.
\begin{table}[]
	\centering
	\caption{Comparison of PSNR and SNR results on on the MNIST and the COIL-20 databases.}
	\label{denoise-Eval}
	\begin{tabular}{|c|c|c|c|c|c|c|c|c|}
		\hline
		\multirow{3}{*}{Methods} & \multicolumn{4}{c|}{MNIST}                                              & \multicolumn{4}{c|}{COIL-20}                                            \\ \cline{2-9}
		& \multicolumn{2}{c|}{Gaussian(dB)} & \multicolumn{2}{c|}{Occlusion(dB)} & \multicolumn{2}{c|}{Gaussian(dB)} & \multicolumn{2}{c|}{Occlusion(dB)} \\ \cline{2-9}
		& PSNR            & SNR             & PSNR             & SNR              & PSNR            & SNR             & PSNR             & SNR              \\ \hline
		OP                       & 21.32           & 13.24           & 19.29            & 10.17            & 27.87           & 19.13           & 26.25            & 19.14            \\ \hline
		KPCA                     & 24.30           & 16.27           & 21.28            & 14.23            & 29.19           & 22.16           & 27.22            & 21.15            \\ \hline
		GRPCA                    & 22.34           & 14.22           & 18.31            & 10.20            & 26.20           & 20.21           & 24.55            & 18.43            \\ \hline
		Ours(Gaussian)           & \textbf{39.03}  & \textbf{29.25}  & \textbf{37.21}   & \textbf{27.26}   & \textbf{39.97}  & \textbf{33.15}  & \textbf{38.22}   & \textbf{31.12}   \\ \hline
		Ours(Polynomial)         & 36.53           & 26.77           & 34.08            & 23.76            & 32.96           & 25.03           & 33.03            & 23.55            \\ \hline
	\end{tabular}
\end{table}


Table~\ref{denoise-Eval} compares our model against all $3$ competing methods. We report the results of our proposed method with the Gaussian Kernel Function ($\gamma = 0.003$) and the Polynomial Kernel Function ($d = 1.5$) in the last two lines. Since the rotated data owns a structure of highly nonlinearity, our method, which combines the data independent kernel trick and kernel low-rank pursuit, consistently outperforms other methods and obtains the highest PSNR and SNR. Fig.~\ref{denoised} visualizes the denoised results of all the methods. It can be seen that our model is more robust to the gross corruption and achieves better recovery results in details due to the unrestraint of rank in the original space. We notice that the proposed method with the Gaussian Kernel obtains a better recovery results than using the Polynomial Kernel. This is because, for big order $d$, the first term in problem (\ref{sub1expHQ}) dominates the optimization procedure, however, with small $d$, the proposed method cannot capture the underlying non-linear structure of data. Contrary to the Polynomial Kernel, the Gaussian Kernel is a infinite-dimensional mapping with bounded value. Thus when the data owns a highly nonlinear structure,  the Gaussian Kernel can perform better than Polynomial Kernel. Comparing with other methods, the results of ours, again, confirm the superiority of combing the implicitly low-rank pursuit and data independent kernel trick.

\begin{table*}[tp]
	\centering
	\setlength{\tabcolsep}{3pt}
	\renewcommand\arraystretch{1.2}
	\caption{Processing time comparison on the MNIST dataset with Gaussian noise.}
	\label{tab:time}
	\begin{tabular}{|c|c|c|c|c|}
		\hline
		Methods     & OP & KPCA   & GRPCA & Ours~(Gaussian kernel)
		\\ \hline
		MNIST~(Gaussian noise)   & 43.9s       &  120.4s     & 131.2s  & 106.1s
		\\ \hline
	\end{tabular}
\end{table*}

The CPU time for all the competing methods on the MNIST dataset with Gaussian noise is presented in tabel~\ref{tab:time}. All methods based on graph or kernel have high computational complexity.
By utilizing the APG strategy for optimization,
our method is faster than the other two graph based methods.

\section{Conclusion and Future Work}
This paper shows a method, more robust than other kernel methods, for solving the non-linear matrix recovery problem. To solve the nonconvex optimization problem, we propose an algorithm that leverages the techniques of linearization and proximal gradient. In the meantime, we also analyze the convergence and complexity of our algorithm, and theoretically prove that the obtained solution is a stationary point.
Compared with the state-of-the-art methods,
our proposed method achieves much better results in both data recovery and denoising tasks.
\par
For future work, we hope to reduce the computational complexity of the proposed algorithm.
It is worth mentioning that the computational complexity of updating $\mathbf{A}$ can be reduced to $O(n^2)$, since the high cost of our method comes from a sequential updating step of the columns of matrix $\mathbf{A}$.
In practice, when optimizing $\mathbf{a}_i$, other columns are fixed and we update the full $\mathbf{A}$ when all columns are calculated. Hence a parallel strategy can be introduced, namely, we can parallelly update the columns of $\mathbf{A}$  in one iteration.
In this case,
the computational complexity of updating $\mathbf{A}$ can be reduced to $O(n^2)$.

\section*{Acknowledgment}
The authors would like to thank the anonymous reviewers for their helpful comments. The work of Guangcan Liu is supported in part by National Natural Science Foundation of China (NSFC) under Grant 61622305 and Grant 61502238, in part by the Natural Science Foundation of Jiangsu Province of China (NSFJPC) under Grant BK20160040. The work of Jun Wang is supported in part by NSFC under Grant 61402224, 6177226), the Fundamental Research Funds for the Central
Universities (NE2014402, NE2016004).

\section*{References}
\bibliographystyle{unsrt}
\bibliography{egbib}


\end{document}